\documentclass[twoside]{article}

\usepackage[accepted]{aistats2020}
\usepackage[round]{natbib}

\usepackage{hyperref}
\usepackage{url}
\usepackage{booktabs}
\usepackage{amsmath, amsthm, amssymb,multirow,hyperref, enumitem, tcolorbox}
\usepackage{array,graphicx,footnote,url}
\usepackage{array,amsmath,amssymb}
\usepackage{epsfig,psfrag,url,verbatim, multirow}
\usepackage{graphics}
\usepackage{graphicx}
\usepackage{xargs}
\usepackage{amsfonts}
\usepackage{color}
\usepackage{float}

\newtheorem{theorem}{Theorem}
\newtheorem{example}{Example}
\newtheorem{lemma}{Lemma}

\newtheorem{definition}{Definition}

\DeclareMathOperator*{\argmin}{arg\,min}

\newcommand\eqdef{\mathrel{\overset{\makebox[0pt]{\mbox{\normalfont\tiny\sffamily def}}}{=}}}
\DeclareMathOperator{\sign}{sign}

\usepackage{subfig}
\begin{document}

%

%

\twocolumn[

\aistatstitle{Learning Hierarchical Interactions at Scale: A Convex Optimization Approach}

\aistatsauthor{Hussein Hazimeh \And Rahul Mazumder}
\aistatsaddress{ Massachusetts Institute of Technology \And Massachusetts Institute of Technology}
]

\begin{abstract}
In many learning settings, it is beneficial to augment the main features with pairwise interactions. Such interaction models can be often enhanced by performing variable selection under the so-called \textsl{strong hierarchy}  constraint: an interaction is non-zero only if its associated main features are non-zero. Existing convex optimization-based algorithms face difficulties in handling problems where the number of main features $p \sim 10^3$ (with total number of features $\sim p^2$). In this paper, we study a convex relaxation which enforces strong hierarchy and develop a highly scalable algorithm based on proximal gradient descent. We introduce novel screening rules that allow for solving the complicated proximal problem in parallel. In addition, we introduce a specialized active-set strategy with gradient screening for avoiding costly gradient computations. The framework can handle problems having dense design matrices, with $p = 50,000$ ($\sim 10^9$ interactions)---instances that are much larger than state of the art. Experiments on real and synthetic data suggest that our toolkit \textsl{hierScale} outperforms the state of the art in terms of prediction and variable selection and can achieve over a 4900x speed-up.
\end{abstract}

\section{Introduction}
In many machine learning applications, augmenting main effects with pairwise interactions can lead to better statistical models. Given a response vector $y \in \mathbb{R}^n$ and data matrix $X = [X_1, X_2, \dots, X_p] \in \mathbb{R}^{n \times p}$, we consider a linear model of the form:
\begin{align} \label{eq:mainmodel}
y = \beta_0 + \sum\nolimits_i X_i \beta_i + \sum\nolimits_{i < j} \theta_{ij} (X_i * X_j) + \epsilon,
\end{align}
where $*$ denotes element-wise multiplication and $\epsilon$ is a noise vector. Above, $\beta_0$ is the intercept, $\sum_i X_i \beta_i$ corresponds to the main effects and
$\sum_{i < j} \theta_{ij} (X_{i} * X_{j})$ denotes the interaction effects.
The goal here is to learn the coefficients $\beta,\theta$. For small $n$, this quickly leads to an ill-posed problem as the total number of coefficients is $1 + p + {p \choose 2}$, which can far exceed $n$. Thus, imposing sparsity can be beneficial from both the statistical and computational viewpoints. While vanilla sparsity-inducing regularization methods (e.g., using the $\ell_0$ or $\ell_1$ norm) can help, structured sparsity can be much more effective in this setting. Particularly, we consider enforcing sparsity under the \textsl{strong hierarchy} (SH) constraint \citep{mccullagh1989generalized, bien2013lasso}, which states that an interaction term should be non-zero only if its corresponding main effect terms are both non-zero. SH can be expressed via the following combinatorial statement:
$$\text{{\bf{Strong Hierarchy}} (SH):}~~\theta_{ij} \neq 0 {\implies} \beta_i \neq 0 \ \land \ \beta_j \neq 0$$
SH is a natural property that is widely used in high-dimensional statistics: it leads to more interpretable models with good predictive performance~\citep{cox1984interaction, mccullagh1989generalized, bien2013lasso}. 
{Moreover, SH promotes \textsl{practical sparsity}, i.e., it reduces the number of main features that need to be measured when making new predictions---this can significantly save on future data collection costs \citep{bien2013lasso}.} 

An impressive line of work for learning sparse interactions under SH is based on a regularization framework (see~\cite{choi2010variable, radchenko2010variable, bien2013lasso, lim2015learning, yan2017hierarchical, she2018group} and the references therein). These methods minimize the empirical risk with sparsity-inducing regularizers to enforce SH. While these lead to estimators with good statistical properties, computation remains a major challenge. 
Indeed, the complex nature of the regularizers used to enforce SH prevents most current optimization algorithms from scaling beyond $p \sim 10^3$. However, in many real-world applications $p$ can be in the order of tens of thousands, which is limiting the adoption of SH methods in practice.
To address this shortcoming, we propose a convex regularization framework and develop a novel scalable algorithm, by carefully exploiting 
the problem-specific structural sparsity.
Our algorithm can solve the SH learning problem with $p=50,000$ ($\sim 10^9$ interactions) and achieves significant speed-ups (over $4900$x) compared to the state of the art. 


\subsection{Problem Formulation} 
Various convex regularizers (a.k.a. penalties) for enforcing SH exist in the literature. In many cases, the choices seem somewhat ad hoc \textcolor{black}{and involve auxiliary variables that can increase the memory and computational footprints}. Here we transparently derive our regularizer via a convex relaxation of a mixed integer program (MIP) \citep{bertsimas1997introduction} that enforces SH. In what follows, we denote the interaction column $X_i * X_j $ by $\tilde{X}_{ij}$. Performing variable selection under SH for model \eqref{eq:mainmodel} can be naturally expressed using $\ell_0$ regularization:
\begin{align} \label{eq:problemInitial}
\min_{\beta, \theta}~~ & f(\beta, \theta) + \alpha_1 \| \beta \|_0 + \alpha_2 \| \theta \|_0  \quad \\ & \text{s.t.} ~~ \theta_{ij} \neq 0 \implies \beta_i \neq 0 \text{ and } \beta_j \neq 0 \nonumber
\end{align}
where $f(\beta, \theta) = \frac{1}{2} \| y - X \beta - \sum_{i < j} \tilde{X}_{ij} \theta_{ij} \|_2^2$ and $\| u \|_0$ denotes the number of non-zeros in the vector $u$. The parameters $\alpha_1$ and $\alpha_2$ control the number of non-zeros. Note that we ignore the intercept term to simplify the presentation. 

\textcolor{black}{Problem \eqref{eq:problemInitial} can be expressed using the following MIP\footnote{There can be pathological cases where the MIP does not satisfy SH. However, when $y$ is drawn from a continuous distribution, SH is satisfied w.p. 1. We discuss this in more detail in Section 1 of the appendix.}, in which we introduce auxiliary binary variables to model the SH constraint and the $\ell_0$ (pseudo) norms:}
\begin{align} \label{eq:MIP}
\begin{split}
\min_{\beta, \theta, z}~~ & f(\beta, \theta) + \alpha_1 \sum\nolimits_i z_i + \alpha_2 \sum\nolimits_{i < j} z_{ij} \\ 
\text{s.t.}~~& |\beta_i| \leq M z_i , \ \forall \  i \quad \\ \quad  & 
|\theta_{ij}| \leq M z_{ij}, ~~ z_{ij} \leq z_i, ~~ z_{ij} \leq z_j, \ \forall \ i < j \\
& z_i \in \{0,1\} \ \forall \ i, ~~ z_{ij} \in \{0,1\} \ \forall \ i < j
\end{split}
\end{align}
\textcolor{black}{where the optimization variables are $\beta$, $\theta$, and $z$}. In the above, $M$ is a large constant chosen such that some optimal solution $\beta^{*}, \theta^{*}$ to \eqref{eq:problemInitial} satisfies $\| \beta^{*}, \theta^{*} \|_{\infty} \leq M$. Here $z_i = 0$ implies $\beta_i = 0$ and similarly $z_{ij} = 0$ implies $\theta_{ij} = 0$. Problem \eqref{eq:MIP} is known to be NP-Hard \citep{natarajan1995sparse} and can be very difficult to scale. Thus, our focus will be on solving a convex relaxation. For easier presentation, we introduce some notation. For every $i \in \{ 1,2, \dots, p \}$, we define $G_i$ as the set of indices of all interactions corresponding to $\beta_{i}$, i.e.,
$$G_i \eqdef \{ (1,i), (2,i), \dots, (i-1,i), (i,i+1), \dots, (i,p) \}.$$
We use the notation $\theta_{G_i}$ to refer to the vector of $\theta_{ij}$'s whose indices are in $G_i$. 



\textcolor{black}{In Lemma \ref{lemma:relaxation}, we derive a convex relaxation of Problem \eqref{eq:MIP}. Specifically, we relax the binary variables in Problem \eqref{eq:MIP} to $[0,1]$, and then simplify the problem to remove the dependence on the $z_i$'s and $z_{ij}$'s, i.e., we move from the extended $(\beta, \theta, z)$ space back to the original $(\beta, \theta)$ space. The resulting relaxation involves box constraints on the coefficients $\beta, \theta$---we eliminate these constraints to simplify the problem.}

\begin{lemma}{(Convex Relaxation)}\label{lemma:relaxation}
\textcolor{black}{The following is a convex relaxation of Problem \eqref{eq:MIP}:}
\begin{align} \label{eq:relaxation}
\min_{\beta,\theta}~~~ f(\beta, \theta) + \Omega(\beta, \theta),
\end{align}
where
$$
\Omega(\beta, \theta) \eqdef \lambda_1 \sum_{i=1}^{p} \max \{ |\beta_i|, \| \theta_{G_i} \|_{\infty} \} + \lambda_2 \| \theta \|_1,$$
and
$\lambda_1 = {\alpha_1}/{M}, \lambda_2 = {\alpha_2}/{M}$.
\end{lemma}
The focus of our paper is on solving \eqref{eq:relaxation}. 
 We note that the relaxation in \eqref{eq:relaxation} belongs to the original $(\beta, \theta)$ space. From an algorithmic perspective, this space is easier to work with compared to an extended space (e.g., the one involving $z$).
\citet{she2018group} propose a general family of problems for enforcing SH, which includes our relaxation above as a special case \textcolor{black}{(however, they do not discuss the connections to the original Problem \eqref{eq:problemInitial})}. The solutions of \eqref{eq:relaxation} satisfy SH with probability one under model \eqref{eq:mainmodel} when $\epsilon \sim N(0,\sigma^2 I)$ (see \cite{she2018group}). For a general response $y$, there can be pathological cases where SH is not satisfied by \eqref{eq:relaxation}, however, these cases rarely appear in practice and are usually not considered by 
current regularization-based approaches for SH \citep{radchenko2010variable, bien2013lasso, lim2015learning, she2018group}. 


\subsection{Contributions}
The main contribution of our work is developing a scalable algorithm for solving \eqref{eq:relaxation}. Our proposal is based on proximal gradient descent (PGD). However, PGD is limited by two major computational bottlenecks, due to the scale of the problem. First, the proximal problem does not admit a closed-form solution, so solving it requires running an iterative optimization algorithm over $\mathcal{O}(p^2)$ variables. Second, the repeated gradient computations can be very expensive as each computation requires $\mathcal{O}(n p^2)$ operations.

To mitigate these bottlenecks, we \textbf{(i)} introduce new \textsl{proximal screening rules} that can efficiently identify many of the zero variables and groups in the proximal problem, \textbf{(ii)} demonstrate how our proposed screening rules can decompose the proximal problem so that it can be solved in parallel, and \textbf{(iii)} develop a specialized active-set algorithm along with a novel \textsl{gradient screening} method for avoiding costly gradient evaluations. We open source the implementation through our toolkit hierScale\footnote{\url{https://pypi.org/project/hierScale}}. Moreover, we demonstrate how our algorithm scales to high-dimensional problems having dense matrices with $p = 50,000$ ($\sim10^{9}$ interactions), achieving over $4900$x speed-ups compared to the state of the art.


\subsection{Related Work}
Many methods for enforcing SH exist in the literature. The methods can be broadly categorized into multi-step methods (e.g., \citet{wu2010screen, hao2014interaction}), Bayesian and approximate methods (e.g., \citet{chipman1996bayesian, thanei2018xyz}), and optimization-based methods. We discuss relevant work 
in the latter category as they are directly related to our work. \citet{choi2010variable} re-parameterize the interactions problem so that $\theta_{ij} = \gamma_{ij} \beta_i \beta_j$ (where $\gamma_{ij}$ is an optimization variable) and enforce sparsity by adding an $\ell_1$ norm regularization on $\beta$ and $\theta$---this approach, however, leads to a challenging non-convex optimization problem. \citet{radchenko2010variable} enforce SH by using $\ell_2$ regularization on the predictions made by every group. \citet{bien2013lasso} propose the Hierarchical Lasso, which shares a similar objective with our problem, except that $\| \theta_{G_i} \|_1$ is used instead of $\| \theta_{G_i} \|_{\infty}$. They use ADMM~\citep{boyd2011distributed} for computation. \citet{lim2015learning} propose an overlapped group Lasso formulation and solve it using a variant of the FISTA algorithm \citep{beckfista} along with strong screening rules \citep{tibshirani2012strong}. Their toolkit glinternet seems to be the fastest toolkit for learning sparse interactions we are aware of. \citet{she2018group} consider a formulation similar to \eqref{eq:relaxation}, but with the $\ell_{\infty}$ norm replaced with $\ell_{2}$ norm, and develop prediction error bounds and a splitting-based algorithm---the largest problem they consider has $p=1000$. 

\citet{mairal2011convex} consider learning problems regularized with sum of $\ell_{\infty}$ norms over groups (with potential overlaps), which includes our problem as a special case. They show that the proximal operator can be efficiently solved using network flow algorithms and propose algorithms based on PGD and ADMM. Our approaches differ: here we exploit the specific structure of our objective function (particularly the presence of both the $\ell_{\infty}$ and $\ell_1$ norms) to derive the proximal screening rules and decompose the proximal problem---such screening/decomposition rules are not discussed in~\cite{mairal2011convex}. 
We also note that \citet{jenatton2011proximal} develop a scalable PGD-based algorithm for problems regularized with sums of $\ell_{\infty}$ or $\ell_2$ norms, under a tree structure constraint. However, our model does not satisfy this constraint and thus their algorithms are not applicable. Many other works also consider algorithms for structured sparsity and discuss interesting connections to submodularity (e.g., see \cite{bach2009exploring, bach2010structured, bach2013learning} and the references therein). Unfortunately, prior work cannot easily scale beyond $p$ in the hundreds to few thousands. A main reason is that the standard algorithms (e.g., PGD and ADMM) are limited by costly gradient evaluations and by solving expensive sub-problems (e.g., solving the proximal problem in PGD requires an iterative optimization method). Our proposal addresses these key computational bottlenecks.



{\bf Notation and supplement:} We use the notation $[p]$ to refer to the set $\{ 1, 2, \dots, p \}$. We denote the complement of a set $A$ by $A^c$. 
For a set $A \subseteq [p]$, $\beta_A$ refers to the sub-vector of $\beta$ restricted to coordinates in $A$.
We use $\nabla_{\beta_A, \theta_B} f(\beta, \theta)$ to refer to the components of $\nabla f(\beta, \theta)$ corresponding to the vectors $\beta_A$ and $\theta_B$. For any scalar $a$, we define $[a]_{+} = \max\{a, 0\}$. 
A function $u \mapsto g(u)$ is said to be Lipschitz with parameter $L$ if $\| g(u) - g(v) \|_2 \leq L \| u - v \|_2$ for all $u,v$. Proofs of all lemmas and theorems are in the supplementary file. 
\section{Proximal Screening and Decomposition} \label{section:PGD}
To solve the non-smooth convex problem~\eqref{eq:relaxation}, we use PGD~\citep{beckfista, nesterov2013gradient}, which is an effective and popular choice for handling structured sparse learning problems (e.g., see \cite{bach2010structured,mairal2011convex, chen2012smoothing} and references therein). However, there are two major bottlenecks: (i) solving the proximal problem which does not admit a closed-form solution and (ii) the repeated gradient computations each requiring $\mathcal{O}(np^2)$ operations. In this section, we address bottleneck (i) through new proximal screening and decomposition rules, and we handle (ii) in Section \ref{section:screening} using active-set updates and gradient screening.

We first present the basic PGD algorithm below. 
\begin{itemize}[leftmargin=*]
\item[] \textbf{Algorithm 1: Proximal Gradient Descent}
\item Input: $\beta^{0}, \theta^{0}$ and $L$: the Lipschitz parameter of the gradient map $(\beta, \theta) \mapsto \nabla f(\beta, \theta)$.
\item \textbf{For} $k\geq 0$ repeat the following till convergence:         \begin{align}
            \tilde{\beta} \gets \beta^{k} - \nabla_{\beta} f(\beta^{k}, \theta^{k})/L, ~~\tilde{\theta} \gets \theta^{k} -&  \nabla_{\theta} f(\beta^{k}, \theta^{k})/L  \nonumber \\
              \begin{bmatrix} \beta^{k+1} \\ \theta^{k+1} \end{bmatrix} \gets \argmin_{\beta, \theta} \frac{L}{2} \Bigg \| \begin{bmatrix} \beta - \tilde{\beta} \\ \theta - \tilde{\theta} \end{bmatrix} \Bigg \|_2^2 & +  \Omega(\beta, \theta)  \label{eq:proximalupdate}
              \end{align}
\end{itemize}
The sequence $\{\beta^{k}\}$ generated by Algorithm 1 is guaranteed to converge to an optimal solution ~\citep{beckfirstorder}. The objective values converge at a rate of $\mathcal{O}(1/k)$, and this can be improved to $\mathcal{O}({1}/{k^2})$ by using accelerated PGD \citep{nesterov2013gradient,beckfirstorder}. However, there is no closed-form solution for the proximal problem in \eqref{eq:proximalupdate}---this is due to the overlapping variables in the $\ell_{\infty}$ norms. Thus, iterative optimization algorithms are needed to solve \eqref{eq:proximalupdate}. For example, \citet{mairal2011convex} presents a dual reformulation and an efficient network flow algorithm for solving a class of proximal problems which includes \eqref{eq:proximalupdate}. In the appendix, we present an alternative dual which uses less variables (as we exploit the specific structure of our problem) and present a dual block coordinate ascent (BCA) algorithm for solving it. However, iterative algorithms (e.g., \citet{mairal2011convex}'s or our proposed BCA) require significant time to solve \eqref{eq:proximalupdate} when $p$ is large, which can lead to a serious bottleneck. 

\textcolor{black}{In what follows, we present algorithm-agnostic schemes to speed up solving the proximal problem. In Section \ref{section:proxscreening}, we introduce screening rules to eliminate variables from the proximal problem (prior to optimization). In Section \ref{section:decomposition}, we demonstrate how these rules can decompose the proximal problem, which allows for solving it in parallel.}
\subsection{Proximal Screening} \label{section:proxscreening}
Typically, we expect the solutions of \eqref{eq:proximalupdate} to be sparse, as $\Omega(\beta, \theta)$ incorporates sparsity-inducing norms. To exploit this sparsity, we propose new \textsl{proximal screening} rules, which can efficiently identify many of the zero groups and variables in \eqref{eq:proximalupdate}. We present the rules in Theorem \ref{theorem:proximalscreening}.
\begin{theorem}(Proximal Screening) \label{theorem:proximalscreening}
Let $\beta^{*}, \theta^{*}$ be the optimal solution of Problem~\eqref{eq:proximalupdate}. Then, the following \textsl{group-level rule} holds for every $i\in [p]$:
\begin{align} 
  \sum_{t \in G_i} \Big[| \tilde{\theta}_{t} | - \frac{\lambda_2}{L} \Big]_{+} \leq \frac{\lambda_1}{L} - | \tilde{\beta}_i | \implies \beta^{*}_i, \theta^{*}_{G_i} = 0, \label{eq:proximalscreening1}
  \end{align}
  and the following \textsl{feature-level rule} holds for $i < j$:
  \begin{align} \label{eq:proximalscreening2}
  |\tilde{\theta}_{ij}| \leq \frac{\lambda_2}{L} \implies \theta^{*}_{ij} = 0.
\end{align}
\end{theorem}
The rules in Theorem \ref{theorem:proximalscreening} can be used to optimize \eqref{eq:proximalupdate} over a smaller set of variables (i.e., only over the variables that do not pass the screening checks). These rules are easy to check. Particularly, the rule in \eqref{eq:proximalscreening1} requires $\mathcal{O}(p)$ operations to screen a group of $p$ variables. The feature-level rule allows us to set $\theta_{ij}$'s with $|\tilde{\theta}_{ij}| \leq \lambda_2/L$ to zero. The group-level rule is less restrictive: in group $i$, the $\theta_{ij}$'s with $|\tilde{\theta}_{ij}| > \lambda_2/L $ can be still set to zero if $| \tilde{\beta}_i |$ is sufficiently small (i.e., if the contribution of main effect $i$ is weak). 

\textcolor{black}{\textbf{Related Work on Screening Rules:} Our proposed rules are safe in the sense that only variables that are zero in the optimal solution of \eqref{eq:proximalupdate} can be discarded. However, our rules are different from the \textsl{safe screening rules} used in the literature (e.g., \cite{ghaoui2010safe, wang2013lasso, bonnefoy2015dynamic, lee2014screening, wang2014two, ndiaye2016gap, nakagawa2016safe, ndiaye2017gap}) that are designed to identify zero variables in the full problem. In particular, our rules can potentially eliminate more variables in the proximal problem compared to the safe rules, since a variable can be safe to eliminate from a proximal problem but not from the full problem. This allows for exploiting more parallelism when solving the proximal problem (see Theorem \ref{theorem:decomp}). Moreover, the state-of-the-art safe rules (e.g., the sequential rules of \citet{wang2013lasso}, and the dynamic and Gap rules of \citet{bonnefoy2015dynamic} and \citet{ndiaye2017gap}), update (improve) the rules as the algorithm progresses, by leveraging previous solutions. However, every such rule update entails a full gradient computation, which can be very costly in our problem (see \cite{ndiaye2017gap} for a survey and a discussion on this computational issue). On the other hand, our rules do not require gradient computations. In our experiments, we compare against glinternet \citep{lim2015learning} which uses strong screening rules (an approximate and aggressive variant of safe rules proposed in~\cite{tibshirani2012strong}).}
\subsection{Proximal Decomposition} \label{section:decomposition}
An important consequence of Theorem~\ref{theorem:proximalscreening} is that it usually decomposes Problem~\eqref{eq:proximalupdate} into many independent smaller optimization problems. This allows solving~\eqref{eq:proximalupdate} in parallel. Before presenting the decomposition formally, we give a simple motivating example. 
\begin{example}
Suppose $p=2$ with one interaction effect; and rule~\eqref{eq:proximalscreening2} has identified $\theta^{*}_{12} = 0$. We can now eliminate $\theta_{12}$ and solve the proximal problem \eqref{eq:proximalupdate} with $\Omega(\beta, \theta) = \lambda_1 |\beta_1| + \lambda_1 |\beta_2|$. This decomposes the problem into two independent optimization tasks involving $\beta_1$ and $\beta_2$---the solutions can be easily obtained via soft-thresholding.
\end{example}
Next, we formalize the idea of decomposition. Let us define $\mathcal{V}$ as the set of indices of the groups that do not pass the screening test in \eqref{eq:proximalscreening1}, i.e.,
\begin{align} \label{eq:V}
    \mathcal{V} = \Big\{ i \in [p] ~\  \Big| \sum_{t \in G_i} \Big[| \tilde{\theta}_{t} | - \frac{\lambda_2}{L} \Big]_{+} > \frac{\lambda_1}{L} - | \tilde{\beta}_i | \Big\}.
\end{align}
We also define $\mathcal{E}$ as the set of interaction indices that do not pass the test in \eqref{eq:proximalscreening2} and whose corresponding main indices are in $\mathcal{V}$, i.e.,
\begin{align} \label{eq:E}
    \mathcal{E} =  \Big\{  (i,j) \in \mathcal{V}^2 ~ \ | \  |\tilde{\theta}_{ij}| > \lambda_2/L \Big\}.
\end{align}

\begin{definition}{(Connected Components)} Define the simple undirected graph $\mathcal{G} = (\mathcal{V},\mathcal{E})$ where the vertex set $\mathcal{V}$ is defined in \eqref{eq:V} and the edge set $\mathcal{E}$ is defined in \eqref{eq:E}. Let the $\kappa$ connected components of $\mathcal{G}$ be denoted by $\{{\mathcal G}_{l}\}_{1}^\kappa$, where $\mathcal{G}_l =(\mathcal{V}_l, \mathcal{E}_l),~l \in [\kappa]$.
\end{definition}
Theorem~\ref{theorem:decomp} states that Problem \eqref{eq:proximalupdate} can be solved by decomposing it into smaller independent optimization problems, each corresponding to a connected component of the graph $\mathcal{G}$.
\begin{theorem} \label{theorem:decomp}
Let $\beta^{*}, \theta^{*}$ be the optimal solution of the proximal problem in \eqref{eq:proximalupdate}. Then, $\beta_{\mathcal{V}^c}^{*} = 0,~~\theta_{\mathcal{E}^c}^{*} = 0$, and for every $l \in [\kappa]$:
\begin{align*} 
&\begin{bmatrix} \beta^{*}_{\mathcal{V}_l} \\ \theta^{*}_{\mathcal{E}_l} \end{bmatrix} = \argmin_{ \beta_{\mathcal{V}_l}, \theta_{\mathcal{E}_l}} \frac{L}{2} \Bigg \| \begin{bmatrix} \beta_{\mathcal{V}_l} - \tilde{\beta}_{\mathcal{V}_l} \\ \theta_{\mathcal{E}_l} - \tilde{\theta}_{\mathcal{E}_l} \end{bmatrix} \Bigg \|_2^2 +  \Omega(\beta_{\mathcal{V}_l}, \theta_{\mathcal{E}_l}).
\end{align*}
\end{theorem}

Theorem 2 allows for solving \eqref{eq:proximalupdate} in parallel. The extent of parallelism depends on the number of the connected components and their sizes. In practice, we solve the problem for a regularization path with warm starts\footnote{A regularization path is the sequence of solutions obtained by solving \eqref{eq:relaxation} for a sequence of $\lambda_1$'s and $\lambda_2$'s. The solution of the current $\lambda_1, \lambda_2$ is used as a warm start (initial solution) when solving for the next $\lambda_1, \lambda_2$ in the sequence.}, along with active-set updates (discussed in Section \ref{section:screening}). These can significantly reduce the sizes of the connected components. Indeed, our experiments on real high-dimensional datasets (with $p \sim 5000$), indicate that the maximum number of vertices and edges in each connected component is in the order of hundreds to few thousands (for all solutions in the path)---see the supplementary for details. The majority of the connected components are typically isolated (i.e., composed of a single vertex), so their corresponding optimization problems can be solved by a simple soft thresholding operation. However, we note that, in the absence of warm starts, the number of edges can grow into hundreds of thousands for the same datasets.


\section{Active Sets and Gradient Screening} \label{section:screening}
In Section \ref{section:PGD}, we addressed the bottleneck of solving the proximal problem. In this section, we focus on another major bottleneck: the repeated full gradient computations in PGD. Particularly, we develop an active-set algorithm that exploits the screening rules and decomposition we introduced in Section \ref{section:PGD} to reduce the number of full gradient computations. Moreover, we introduce a new gradient screening method which aids in reducing the cost of every gradient computation.

\subsection{Active Set Updates}

Every evaluation of $\nabla f(\beta, \theta)$ in Algorithm 1 requires $\mathcal{O}(n p^2)$ operations, which can take several minutes on a modern machine when $p \sim 50,000$. To minimize the number of these full gradient evaluations, we propose an active-set algorithm: we run Algorithm 1 over a small subset of variables, namely, the active set. After obtaining an optimal solution restricted to the active set, we augment the set with the variables that violate the optimality conditions (if any) and resolve the problem on the new set. Such an approach is effectively used for speeding up sparse learning algorithms in other contexts \citep{meier2008group, glmnet, morvan2018whinter}. 

Our active set is defined by the sets ${\mathcal A}$ and ${\mathcal T}$
containing indices of the main effects and interaction effects (respectively) to be included in the model.
Given $\mathcal{A}$ and $\mathcal{T}$ we consider
\begin{equation}
\begin{aligned} \label{eq:activesetProblem}
\hat{\beta}, \hat{\theta} \in & \argmin\nolimits_{\beta, \theta} ~~  f(\beta, \theta) + \Omega(\beta, \theta ) \\ & \text{s.t. }~~~\beta_{\mathcal{A}^c} = 0, ~~\theta_{\mathcal{T}^c} = 0, 
\end{aligned}
\end{equation}

which is solved with Algorithm 1 {\em restricted} to the active set.
To check whether $\hat{\beta}, \hat{\theta}$ is optimal for Problem \eqref{eq:relaxation}, we run a single iteration of Algorithm 1 over all the variables (including those outside the active set) -- we refer to this as the \textsl{master iteration}. If the master iteration does not change the support (i.e., the non-zeros), then the solution $\hat{\beta}, \hat{\theta}$ is optimal. Otherwise, we augment $\mathcal{A}$ and $\mathcal{T}$ with the variables that became non-zero (after the iteration) and solve \eqref{eq:activesetProblem} again. To speed up the costly master iteration, we perform screening and decompose the master iteration as described in Theorem \ref{theorem:decomp} so that it can be solved in parallel. We present the algorithm more formally below:
\begin{itemize}[leftmargin=*]
\item[] \textbf{Algorithm 2: Parallel Active-set Algorithm}
\item Input: Initial estimates of $\mathcal{A}$ and $\mathcal{T}$
\item \text{\textbf{Repeat} until convergence:}
\begin{enumerate}
    \item Solve the problem restricted to the active set, i.e., \eqref{eq:activesetProblem} to get a solution $\hat{\beta}, \hat{\theta}$.
    \item Compute $\nabla_{\beta} f(\hat{\beta}, \hat{\theta})$ and $\nabla_{\theta} f(\hat{\beta}, \hat{\theta})$. Set $\tilde{\beta} \gets \hat{\beta} - \frac{1}{L} \nabla_{\beta} f(\hat{\beta}, \hat{\theta}), ~~\tilde{\theta} \gets \hat{\theta} - \frac{1}{L} \nabla_{\theta} f(\hat{\beta}, \hat{\theta})$.
    \item Construct the graph $\mathcal{G} = (\mathcal{V}, \mathcal{E}$) in Definition 1 and find its connected components.
    \item Solve \eqref{eq:proximalupdate} in parallel using Theorem 2. If the support stays the same, \textbf{terminate}, o.w., augment $\mathcal{A}$ and $\mathcal{T}$ with variables that just became non-zero.
\end{enumerate}
\end{itemize}
Steps 2-4 in Algorithm 2 are the equivalent of performing one iteration of Algorithm 1 over all variables, while using screening and decomposition. We note that screening and decomposition are also very useful at the active-set level (i.e., for step 1) as we need to repeatedly solve the proximal problem till convergence. Typically, the active-set sizes are relatively small, which can help further mitigate the cost of solving the proximal problem.
\subsection{Gradient Screening}
Algorithm 2 effectively reduces the total number of gradient computations. However, even a single gradient computation can take minutes for $p \sim 50,000$. Here we propose a novel gradient screening method to reduce the cost of every gradient computation in Algorithm 2. Specifically, every time a gradient is needed, we identify parts of the gradient that are not essential to optimization---this allows for computing a smaller (and cheaper) gradient. This is a major improvement over current active-set approaches, which require a full gradient computation to check optimality. 


We note that the full gradient in step 2 of Algorithm 2 is only used to construct the graph $\mathcal{G}$ in step 3. The next lemma, states that only a part of the gradient is needed to construct $\mathcal{G}$.
\begin{lemma} \label{lemma:gradexplain}
The graph $\mathcal{G}$ in step 2 of Algorithm 2 can be constructed from the following gradients: $\nabla_{\beta} f(\hat{\beta}, \hat{\theta})$, $\nabla_{\theta_{\mathcal{T}}} f(\hat{\beta}, \hat{\theta})$, and $\nabla_{\theta_{\mathcal{S}}} f(\hat{\beta}, \hat{\theta})$, where $\mathcal{S}$ is the \textsl{critical set} defined by
\begin{align} \label{eq:safeset}
\mathcal{S} = \{ (i,j) \in \mathcal{T}^c \ | \ |\nabla_{\theta_{ij}} f(\hat{\beta}, \hat{\theta})| > \lambda_2\}.
\end{align}
\end{lemma}
Note that $\nabla_{\beta} f(\hat{\beta}, \hat{\theta})$ is relatively easy to compute, and $\nabla_{\theta_{\mathcal{T}}} f(\hat{\beta}, \hat{\theta})$ is available as a byproduct of step 1. Thus, if the size of $\mathcal{S}$ in \eqref{eq:safeset} is small, then Lemma~\ref{lemma:gradexplain} suggests a significant reduction in the computation time of step~2.
Specifically, if $\mathcal S$ is given, computing the gradients in Lemma~2 
has a cost $\mathcal{O}(n (p + |\mathcal{S}|))$, which can be much smaller than the cost of
full gradient computation $\mathcal{O}(n p^2)$.
As discussed below, we can obtain ${\mathcal S}$ without explicitly computing $\nabla_{\theta_{ij}} f(\hat{\beta}, \hat{\theta})$
for all $(i,j) \in {\mathcal T}^c$---an operation that 
costs $\mathcal{O}(n p^2)$.

Our key idea is to obtain a set $\hat{\mathcal{S}}$ that is guaranteed to contain $\mathcal{S}$ (i.e., $\mathcal{S} \subset \hat{\mathcal{S}}$), by using currently available information on gradients.
Note that $|\hat{\mathcal{S}}|$ can be larger than $|\mathcal{S}|$, but we will require $|\hat{\mathcal{S}}|$ to be significantly smaller than $p^2$. The next lemma, presents a way to construct $\hat{\mathcal{S}}$ by using the gradient of a solution $\beta^{w}, \theta^{w}$ that was obtained prior to $\hat{\beta}, \hat{\theta}$ (e.g., from a warm start or a previous iteration of Algorithm 2).

\begin{lemma} \label{lemma:screening}
Let $\beta^{w} \in \mathbb{R}^p$ and $\theta^{w} \in \mathbb{R}^{p \choose 2}$ be arbitrary vectors, and let $\mathcal{S}$ be the critical set defined in \eqref{eq:safeset}. Define $\gamma = (X \beta^{w} + \tilde{X} \theta^{w}) - (X \hat{\beta} + \tilde{X} 
\hat{\theta})$ and $C = \max_{i,j} \| \tilde{X}_{ij} \|_2$. 
Then, the following holds:
\begin{align*}
 \mathcal{S} \subseteq \hat{\mathcal{S}} \eqdef \{ (i,j) \in \mathcal{T}^c \ | \ | \nabla_{\theta_{ij}} f(\beta^{w}, \theta^{w}) | > \lambda_2 -  C \|\gamma \|_2   \}.
\end{align*}
\end{lemma}
The set $\hat{\mathcal{S}}$ in Lemma~\ref{lemma:screening} is constructed based on the gradient at a previous estimate $(\beta^{w}, \theta^{w})$; and not at the current point $(\hat{\beta}, \hat{\theta})$.
When the predictions made by the estimators $(\hat{\beta}, \hat{\theta})$ and $(\beta^{w}, \theta^{w})$ are close (i.e., small $\| \gamma \|_2$),  
$\hat{\mathcal{S}}$ is a good estimate of $\mathcal{S}$.


Lemma \ref{lemma:screening} lays the foundation of our \emph{gradient screening} procedure. During the course of Algorithm 2, we always maintain a solution $\beta^{w}, \theta^{w}$ for which we store $|\nabla_{\theta} f(\beta^{w}, \theta^{w})|$. We replace step 2 in Algorithm 2 with the following gradient screening module:
\begin{enumerate}
     \item[]  {\bf{Gradient Screening}} 
    \item Compute $\hat{\mathcal{S}}$ (defined in Lemma \ref{lemma:screening}) and $\nabla_{\theta_{\hat{\mathcal{S}}}} f(\hat{\beta}, \hat{\theta})$ to obtain $\nabla_{\theta_{{\mathcal{S}}}} f(\hat{\beta}, \hat{\theta})$.
    \item Compute $\nabla_{\beta} f(\hat{\beta}, \hat{\theta})$ and obtain $\nabla_{\theta_{\mathcal{T}}} f(\hat{\beta}, \hat{\theta})$. Set $\tilde{\beta} \gets \hat{\beta} - \frac{1}{L} \nabla_{\beta} f(\hat{\beta}, \hat{\theta})$ and $\tilde{\theta}_{ij} \gets \hat{\theta}_{ij} - \frac{1}{L} \nabla_{\theta_{ij}} f(\hat{\beta}, \hat{\theta})$ for every $(i,j) \in \mathcal{T} \cup {\mathcal{S}}$.
    
    
    \item If $|\hat{\mathcal{S}}| > p$, set $(\beta^{w}, \theta^{w}) \gets (\hat{\beta}, \hat{\theta})$ and compute/store $|\nabla_{\theta} f(\beta^{w}, \theta^{w})|$.
\end{enumerate}
The set $\hat{\mathcal{S}}$ can be identified in $\mathcal{O}(\log{p})$ by using a variant of binary search on the sorted entries of $|\nabla_{\theta} f(\beta^{w}, \theta^{w})|$ (the latter can be sorted once at a cost of $\mathcal{O}(p^2 \log p$) and stored). Moreover, computing $\nabla_{\theta_{\hat{\mathcal{S}}}} f(\hat{\beta}, \hat{\theta})$ and obtaining $\nabla_{\theta_{{\mathcal{S}}}} f(\hat{\beta}, \hat{\theta})$ is $\mathcal{O}(n (p + |\hat{\mathcal{S}}|))$. Thus, the complexity of gradient screening can be much smaller than $\mathcal{O}(np^2)$. \textcolor{black}{For gradient screening to yield speed-ups, we need $|\hat{\mathcal{S}}| \ll p^2$ (otherwise, the overall complexity will be similar to that with no gradient screening). Thus, in Step 3 in the above, we update $\beta^{w}, \theta^{w}$ when $|\hat{\mathcal{S}}|$ becomes large (recall that improving the estimate $\beta^{w}, \theta^{w}$ can reduce the size of $|\hat{\mathcal{S}}|$). In particular, Step 3 is performed when $|\hat{\mathcal{S}}| > p$; we choose the threshold $p$ to ensure that $|\hat{\mathcal{S}}|$ stays in the order of $p$ in subsequent iterations\footnote{Other choices are possible, but it is important that the chosen threshold keeps $|\hat{\mathcal{S}}|$ from growing in the order of $p^2$.}.}
The initial $\beta^{w}, \theta^{w}$ can be obtained from a previous solution in the regularization path. In practice, predictions across consecutive solutions in the path are usually close, making $\hat{\mathcal{S}}$ close to $\mathcal{S}$---this explains the multi-fold speed-ups we observe in our experiments.



\section{Experiments} \label{section:experiments}
We study the empirical performance of our algorithm and  compare with popular toolkits for learning interactions under SH: hierNet \citep{bien2013lasso} and glinternet \citep{lim2015learning}; in addition to 
Boosting with trees using XGBoost (which is often used to learn interactions). 
We also use the optimization toolbox SPAMS~\citep{mairal2011convex} to solve \eqref{eq:relaxation} and compare its running time with hierScale.


\noindent {\bf Our Toolkit hierScale:}
We open-sourced our toolkit hierScale (\url{https://pypi.org/project/hierScale/}), written in Python with critical code sections compiled into machine code using Numba \citep{lam2015numba}. hierScale has a low memory footprint (it generates interaction columns on the fly) 
and supports multi-core computation. 

\noindent {\bf Synthetic Data Generation:}
We generate $X_{n \times p}$ to be an iid standard Gaussian ensemble;
and form $y = X \beta^{0} + \tilde{X} \theta^{0} + \epsilon$, where $\epsilon_i \stackrel{\text{iid}}{\sim} N(0,\sigma^2)$ is independent of $X$. For $\beta^{0}$ and $\theta^{0}$, we consider three settings: (I) {Hierarchical Truth}: $\beta^{0}$ and $\theta^{0}$ satisfy SH, (II) {Anti-Hierarchical Truth}: $\theta^{0}_{ij} \neq 0 \implies \beta^{0}_{i} =0, \beta^{0}_{j} =0$, and (III) {Main-Only Truth}: $\theta^{0} = 0$. In (I)--(III), all non-zero coefficients are set to $1$. For (I) and (II), we set $\| \theta^{0} \|_0 = \| \beta^{0} \|_0$. We take $\sigma^2$ such that the signal-to-noise-ratio (SNR) $=\text{Var}( X\beta^{0} + \tilde{X} \theta^{0}) / \sigma^2 = 10$.


\noindent {\bf Timings:}
We compare the running time of hierScale versus hierNet, glinternet, and SPAMS, on both synthetic and real datasets. For SPAMS, we use the same objective function of hierScale, and fit a regularization path (with warm starts) using FISTA. We note that SPAMS is not designed to solve interactions problems, so we had to generate the interaction columns apriori, which limits us to problems where the interactions can fit in memory. We generate the synthetic data under hierarchical truth with $n = 1000$ and $\| \beta^{0} \|_0 = \| \theta^{0} \|_0  = 5$ and consider $p$ up to $50,000$ ($\sim 10^9$ interactions). 
Among real data, we consider the Amazon Reviews dataset \citep{hazimeh2018fast} and use two variants of it: Amazon-1 ($p = 10160$, $n = 1000$) and Amazon-2 ($p = 5000$, $n = 1000$). We also consider the dataset from CoEPrA 2006 (Regression Problem 1) ($p = 5786$, $n=89$)\footnote{Available at \url{http://www.coepra.org}}, and the Riboflavin dataset ($p = 4088, n = 71$) \citep{buhlmann2014high}. For all toolkits, we set the tolerance level to $10^{-6}$, $\lambda_{\text{min}} = 0.05 \lambda_{\text{max}}$ and generate a path with 100 solutions. 
For hierScale and SPAMS, we set $\lambda_2 = 2 \lambda_1$. Computations are carried out on 
a machine with a 12-core Intel Xeon E5 @ 2.7 GHz and 64GB of RAM.

The results are in Table \ref{table:timings}. 
hierScale achieves over a $4900$x speed-up compared to hierNet, 700x speed-up compared to SPAMS, and $690$x speed-up compared to glinternet (e.g., on the Amazon dataset glinternet cannot terminate in a week).
We note that~\citet{she2018group} recently proposed an algorithm for 
a problem similar to ours, but do not provide a public toolkit---the largest problem reported in their paper is for $p=10^3$: In the best case, their method takes $\sim 51$ seconds per solution, whereas hierScale takes 0.2 seconds.

\begin{table*}[!htbp] 
\renewcommand{\tabcolsep}{1.3mm}
\centering
\caption{Average time (s) for obtaining a solution in the regularization path. The symbols * and ** indicate that the toolkit does not terminate in 3 days and 1 week, respectively. The dash (-) indicates a crash due to memory issues. The dot indicates that the data matrix could not fit in memory.}
\label{table:timings}
\begin{tabular}{@{}lcccccccccc@{}}
\toprule
\multicolumn{1}{l}{\multirow{2}{*}[-0.3cm]{Toolkit}} & \multicolumn{6}{c}{Synthetic datasets}              & \multicolumn{4}{c}{Real datasets}     \\ 
\cmidrule(l){2-7} 
\cmidrule(l){8-11} 
\multicolumn{1}{c}{}                         & p=500 & 1000 & 2000 & 5000 & 15000 & 50000 & Ribo. & Coepra & Amazon-1  & Amazon-2 \\ 
\multicolumn{1}{c}{}                         &   &   &   &   &   &   &  (p=4088) &  (p=5786)  & (p=5000) & (p=10160) \\ 
\midrule
hierScale                                    & 0.1 & 0.2  & 0.5  & 3.2  & 26.6   & 730    & 2.2 & 3.3 & 1.7   & 8.7    \\
glinternet                                   & 0.3 & 1    & 3.8  & 24.4 & 226.9  & *      & 2.7  & 7.8      & *    & **      \\
hierNet                                      & 490 & -   & -   & -   & -     & -     & -   & -      & -     & -     \\ 
SPAMS & 17.5 & 67.9 & 351.9 &  .   &   . & .     & 1042.3  & -   & .     & . \\
\bottomrule
\end{tabular}
\end{table*}

\noindent {\bf Variable Selection:}
We compare the False Discovery Rate (FDR) at different sparsity levels. We generate synthetic data with $n=100, p=200$, and $ \| \beta^{0} \|_0 = 10$ and report the FDR averaged over $100$ datasets, in Figure \ref{figure:FDR}. Under hierarchical truth, all the methods perform roughly similarly. However, under anti-hierarchy or main-only truth, our method can perform significantly better. 

\begin{figure*}[!htbp]
\centering
\includegraphics[scale=0.36]{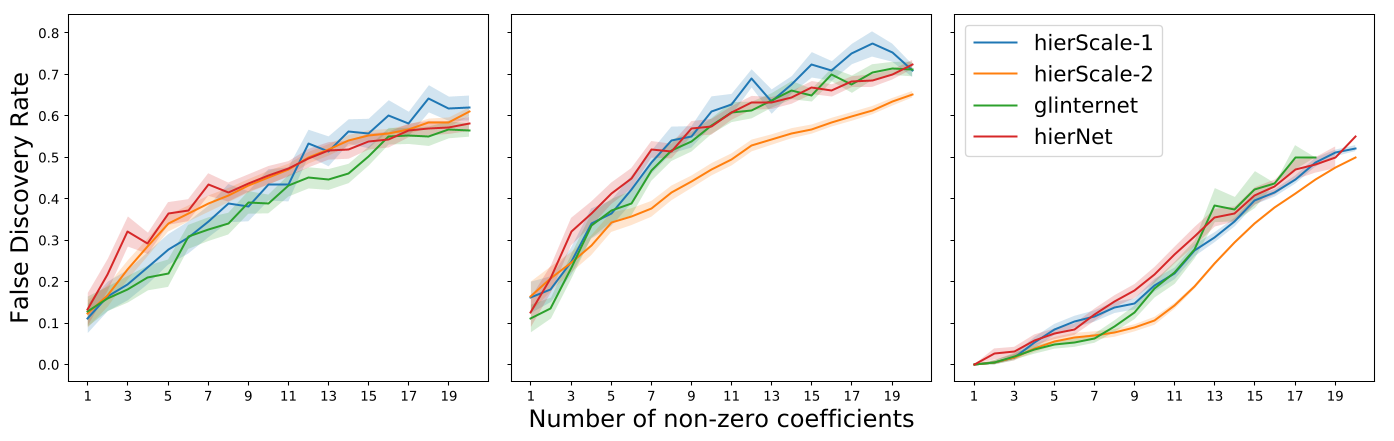} 

\caption{{False Discovery Rate (FDR) for different methods, under 3 settings (I) Hierarchical Truth, (II) Anti-Hierarchical Truth and (III) Main-Only Truth (left to right). hierScale-1 and hierScale-2 refer to our method with $\lambda_2 = \lambda_1$ and $\lambda_2 = 2 \lambda_1$, respectively. The shaded regions correspond to the standard error.  }}
\label{figure:FDR} 
\end{figure*}
\begin{figure*}[!htb]
\vspace{-0.5cm}
    \centering
    \subfloat{{\includegraphics[scale=0.29]{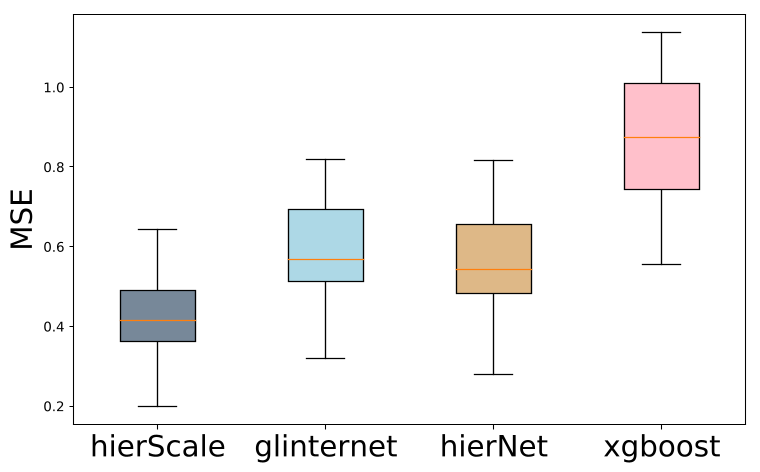}}}
    \quad
    \subfloat{\includegraphics[scale=0.31]{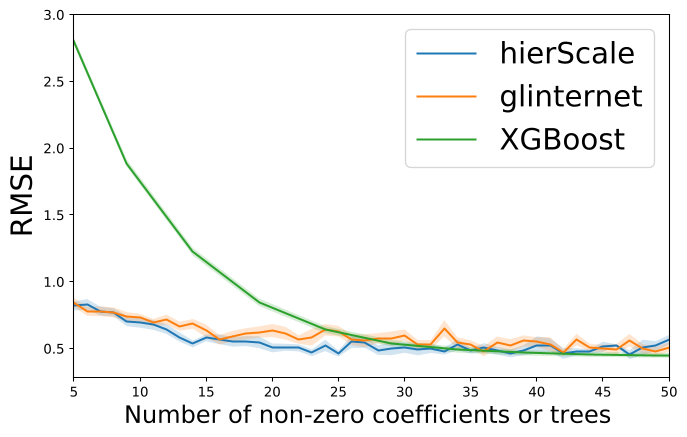}}
    \caption{\textbf{Left}: Test MSE on the synthetic data (Main-Only truth). \textbf{Right}: Test RMSE on the Riboflavin dataset ($n = 50$, $p=4088$). XGBoost is limited to depth 2. The shaded regions represent the standard error.}%
    \label{fig:pred}%
\end{figure*}

\noindent {\bf Prediction Tasks:}
We now compare the prediction performance of hierScale with the competing methods on both synthetic and real datasets. 

\noindent {\em Synthetic data}: We use the same dataset as for variable selection (above). We tune on a separate validation set and report the prediction MSE on the testing data (training, validation, and test sets are of the same size). For hierNet and glinternet, we tune over 50 parameter values. For hierScale we use 50 $\lambda_1$-values and $\lambda_2 \in \{ \lambda_1, 2\lambda_1 \}$. For XGBoost we use 50 tree-sizes (between $20$ and $1000$) and learning rates $\in \{ 0.01, 0.1 \}$. Results across 20 runs (for Main Only Truth) are in 
Figure~\ref{fig:pred}: Our method shows significant improvements in prediction accuracy. The results for hierarchical and anti-hierarchical truth are in the supplementary, with results across methods being comparable. 

\noindent {\em Real data}: We consider the Riboflavin dataset ($p = 4088$, $n=71$)~\citep{buhlmann2014high} for predicting Vitamin $B_2$ production from gene expression levels. We train on $50$ randomly chosen samples and compute the test RMSE on the remaining $21$ samples. To reduce the variability, we repeat this training/testing procedure 30 times and report the average RMSE. We plot the RMSE versus the sparsity level in Figure \ref{fig:pred}. 
For hierScale we vary $\lambda_2 \in \{ \lambda_1, 10 \lambda_1, 100 \lambda_1 \}$, and for XGBoost we vary the learning rate $\in \{0.01, 0.1, 1\}$. Figure~\ref{fig:pred} reports the best RMSE (across the different $\lambda_2$'s for hierScale the learning rates for XGBoost). We see that the SH methods (hierScale and glinternet) can be more effective than boosting at low/moderate sparsity levels (note $n=71$ is small), which is desirable for interpretable learning. On this dataset, hierScale shows marginal improvement over glinternet. \textcolor{black}{We note that \citet{lim2015learning}'s experiments also indicate that SH methods can be more effective than boosted trees in terms of FDR. }

\section{Conclusion and Future Work}
We studied the problem of learning sparse interactions under the strong hierarchy constraint. We introduced a transparent convex relaxation for the problem and developed a scalable algorithm based on proximal gradient descent. 
Our algorithm employs new screening rules for decomposing the proximal problem along with a specialized active-set method and gradient screening for mitigating costly gradient computations. Experiments on real and synthetic data show that our method achieves significant speed-ups over the state of the art and more robust variable selection. 

There are several promising directions for future work. For instance, our proximal and gradient screening methods can be extended to general group-sparsity problems with overlaps between the group norms, which  can potentially speed up the current solvers. 
Moreover, it would be interesting to establish bounds on the  sizes of the connected components and the number of variables eliminated by gradient screening.


\clearpage

\section*{Acknowledgments} We would like to thank Jacob Bien for pointing us to relevant references on hierarchical variable selection. Our work was partially supported by ONR-N000141512342, ONR-N000141812298 (YIP), and NSF-IIS1718258.

\bibliography{ref}

\begin{thebibliography}{42}
\providecommand{\natexlab}[1]{#1}
\providecommand{\url}[1]{\texttt{#1}}
\expandafter\ifx\csname urlstyle\endcsname\relax
  \providecommand{\doi}[1]{doi: #1}\else
  \providecommand{\doi}{doi: \begingroup \urlstyle{rm}\Url}\fi

\bibitem[Bach(2013)]{bach2013learning}
Francis Bach.
\newblock Learning with submodular functions: A convex optimization
  perspective.
\newblock \emph{Foundations and Trends{\textregistered} in Machine Learning},
  6\penalty0 (2-3):\penalty0 145--373, 2013.

\bibitem[Bach(2009)]{bach2009exploring}
Francis~R Bach.
\newblock Exploring large feature spaces with hierarchical multiple kernel
  learning.
\newblock In \emph{Advances in neural information processing systems}, pages
  105--112, 2009.

\bibitem[Bach(2010)]{bach2010structured}
Francis~R Bach.
\newblock Structured sparsity-inducing norms through submodular functions.
\newblock In \emph{Advances in Neural Information Processing Systems}, pages
  118--126, 2010.

\bibitem[Beck(2017)]{beckfirstorder}
Amir Beck.
\newblock \emph{First-Order Methods in Optimization}, volume~25.
\newblock SIAM, 2017.

\bibitem[Beck and Teboulle(2009)]{beckfista}
Amir Beck and Marc Teboulle.
\newblock A fast iterative shrinkage-thresholding algorithm for linear inverse
  problems.
\newblock \emph{SIAM journal on imaging sciences}, 2\penalty0 (1):\penalty0
  183--202, 2009.

\bibitem[Beck and Tetruashvili(2013)]{BeckConvergence}
Amir Beck and Luba Tetruashvili.
\newblock On the convergence of block coordinate descent type methods.
\newblock \emph{SIAM Journal on Optimization}, 23\penalty0 (4):\penalty0
  2037--2060, 2013.
\newblock \doi{10.1137/120887679}.
\newblock URL \url{http://dx.doi.org/10.1137/120887679}.

\bibitem[Bertsekas(2016)]{bertsekas2016nonlinear}
D.P. Bertsekas.
\newblock \emph{Nonlinear Programming}.
\newblock Athena scientific optimization and computation series. Athena
  Scientific, 2016.
\newblock ISBN 9781886529052.
\newblock URL \url{https://books.google.com/books?id=TwOujgEACAAJ}.

\bibitem[Bertsimas and Tsitsiklis(1997)]{bertsimas1997introduction}
Dimitris Bertsimas and John~N Tsitsiklis.
\newblock \emph{Introduction to linear optimization}, volume~6.
\newblock Athena Scientific Belmont, MA, 1997.

\bibitem[Bien et~al.(2013)Bien, Taylor, and Tibshirani]{bien2013lasso}
Jacob Bien, Jonathan Taylor, and Robert Tibshirani.
\newblock A lasso for hierarchical interactions.
\newblock \emph{Annals of statistics}, 41\penalty0 (3):\penalty0 1111, 2013.

\bibitem[Bonnefoy et~al.(2015)Bonnefoy, Emiya, Ralaivola, and
  Gribonval]{bonnefoy2015dynamic}
Antoine Bonnefoy, Valentin Emiya, Liva Ralaivola, and Remi Gribonval.
\newblock Dynamic screening: Accelerating first-order algorithms for the lasso
  and group-lasso.
\newblock \emph{IEEE Transactions on Signal Processing}, 63\penalty0
  (19):\penalty0 5121--5132, 2015.

\bibitem[Boyd et~al.(2011)Boyd, Parikh, Chu, Peleato, and
  Eckstein]{boyd2011distributed}
Stephen Boyd, Neal Parikh, Eric Chu, Borja Peleato, and Jonathan Eckstein.
\newblock Distributed optimization and statistical learning via the alternating
  direction method of multipliers.
\newblock \emph{Foundations and Trends{\textregistered} in Machine learning},
  3\penalty0 (1):\penalty0 1--122, 2011.

\bibitem[B{\"u}hlmann et~al.(2014)B{\"u}hlmann, Kalisch, and
  Meier]{buhlmann2014high}
Peter B{\"u}hlmann, Markus Kalisch, and Lukas Meier.
\newblock High-dimensional statistics with a view toward applications in
  biology.
\newblock \emph{Annual Review of Statistics and Its Application}, 1:\penalty0
  255--278, 2014.

\bibitem[Chen et~al.(2012)Chen, Lin, Kim, Carbonell, and
  Xing]{chen2012smoothing}
Xi~Chen, Qihang Lin, Seyoung Kim, Jaime~G Carbonell, and Eric~P Xing.
\newblock Smoothing proximal gradient method for general structured sparse
  regression.
\newblock \emph{The Annals of Applied Statistics}, 6\penalty0 (2):\penalty0
  719--752, 2012.

\bibitem[Chipman(1996)]{chipman1996bayesian}
Hugh Chipman.
\newblock Bayesian variable selection with related predictors.
\newblock \emph{Canadian Journal of Statistics}, 24\penalty0 (1):\penalty0
  17--36, 1996.

\bibitem[Choi et~al.(2010)Choi, Li, and Zhu]{choi2010variable}
Nam~Hee Choi, William Li, and Ji~Zhu.
\newblock Variable selection with the strong heredity constraint and its oracle
  property.
\newblock \emph{Journal of the American Statistical Association}, 105\penalty0
  (489):\penalty0 354--364, 2010.

\bibitem[Cox(1984)]{cox1984interaction}
David~R Cox.
\newblock Interaction.
\newblock \emph{International Statistical Review/Revue Internationale de
  Statistique}, pages 1--24, 1984.

\bibitem[Friedman et~al.(2010)Friedman, Hastie, and Tibshirani]{glmnet}
Jerome Friedman, Trevor Hastie, and Robert Tibshirani.
\newblock Regularization paths for generalized linear models via coordinate
  descent.
\newblock \emph{Journal of Statistical Software}, 33\penalty0 (1):\penalty0
  1--22, 2010.
\newblock URL \url{http://www.jstatsoft.org/v33/i01/}.

\bibitem[Ghaoui et~al.(2010)Ghaoui, Viallon, and Rabbani]{ghaoui2010safe}
Laurent~El Ghaoui, Vivian Viallon, and Tarek Rabbani.
\newblock Safe feature elimination for the lasso and sparse supervised learning
  problems.
\newblock \emph{arXiv preprint arXiv:1009.4219}, 2010.

\bibitem[Hao and Zhang(2014)]{hao2014interaction}
Ning Hao and Hao~Helen Zhang.
\newblock Interaction screening for ultrahigh-dimensional data.
\newblock \emph{Journal of the American Statistical Association}, 109\penalty0
  (507):\penalty0 1285--1301, 2014.

\bibitem[Hazimeh and Mazumder(2018)]{hazimeh2018fast}
Hussein Hazimeh and Rahul Mazumder.
\newblock Fast best subset selection: Coordinate descent and local
  combinatorial optimization algorithms.
\newblock \emph{arXiv preprint arXiv:1803.01454}, 2018.

\bibitem[Jenatton et~al.(2011)Jenatton, Mairal, Obozinski, and
  Bach]{jenatton2011proximal}
Rodolphe Jenatton, Julien Mairal, Guillaume Obozinski, and Francis Bach.
\newblock Proximal methods for hierarchical sparse coding.
\newblock \emph{Journal of Machine Learning Research}, 12\penalty0
  (Jul):\penalty0 2297--2334, 2011.

\bibitem[Lam et~al.(2015)Lam, Pitrou, and Seibert]{lam2015numba}
Siu~Kwan Lam, Antoine Pitrou, and Stanley Seibert.
\newblock Numba: A llvm-based python jit compiler.
\newblock In \emph{Proceedings of the Second Workshop on the LLVM Compiler
  Infrastructure in HPC}, page~7. ACM, 2015.

\bibitem[Lee and Xing(2014)]{lee2014screening}
Seunghak Lee and Eric~P Xing.
\newblock Screening rules for overlapping group lasso.
\newblock \emph{arXiv preprint arXiv:1410.6880}, 2014.

\bibitem[Lim and Hastie(2015)]{lim2015learning}
Michael Lim and Trevor Hastie.
\newblock Learning interactions via hierarchical group-lasso regularization.
\newblock \emph{Journal of Computational and Graphical Statistics}, 24\penalty0
  (3):\penalty0 627--654, 2015.

\bibitem[Mairal et~al.(2011)Mairal, Jenatton, Obozinski, and
  Bach]{mairal2011convex}
Julien Mairal, Rodolphe Jenatton, Guillaume Obozinski, and Francis Bach.
\newblock Convex and network flow optimization for structured sparsity.
\newblock \emph{Journal of Machine Learning Research}, 12\penalty0
  (Sep):\penalty0 2681--2720, 2011.

\bibitem[McCullagh and Nelder(1989)]{mccullagh1989generalized}
Peter McCullagh and John~A Nelder.
\newblock \emph{Generalized linear models}, volume~37.
\newblock CRC press, 1989.

\bibitem[Meier et~al.(2008)Meier, Van De~Geer, and
  B{\"u}hlmann]{meier2008group}
Lukas Meier, Sara Van De~Geer, and Peter B{\"u}hlmann.
\newblock The group lasso for logistic regression.
\newblock \emph{Journal of the Royal Statistical Society: Series B (Statistical
  Methodology)}, 70\penalty0 (1):\penalty0 53--71, 2008.

\bibitem[Morvan and Vert(2018)]{morvan2018whinter}
Marine~Le Morvan and Jean-Philippe Vert.
\newblock Whinter: A working set algorithm for high-dimensional sparse second
  order interaction models.
\newblock \emph{arXiv preprint arXiv:1802.05980}, 2018.

\bibitem[Nakagawa et~al.(2016)Nakagawa, Suzumura, Karasuyama, Tsuda, and
  Takeuchi]{nakagawa2016safe}
Kazuya Nakagawa, Shinya Suzumura, Masayuki Karasuyama, Koji Tsuda, and Ichiro
  Takeuchi.
\newblock Safe pattern pruning: An efficient approach for predictive pattern
  mining.
\newblock In \emph{Proceedings of the 22nd acm sigkdd international conference
  on knowledge discovery and data mining}, pages 1785--1794. ACM, 2016.

\bibitem[Natarajan(1995)]{natarajan1995sparse}
Balas~Kausik Natarajan.
\newblock Sparse approximate solutions to linear systems.
\newblock \emph{SIAM journal on computing}, 24\penalty0 (2):\penalty0 227--234,
  1995.

\bibitem[Ndiaye et~al.(2016)Ndiaye, Fercoq, Gramfort, and
  Salmon]{ndiaye2016gap}
Eugene Ndiaye, Olivier Fercoq, Alexandre Gramfort, and Joseph Salmon.
\newblock Gap safe screening rules for sparse-group lasso.
\newblock In \emph{Advances in Neural Information Processing Systems}, pages
  388--396, 2016.

\bibitem[Ndiaye et~al.(2017)Ndiaye, Fercoq, Gramfort, and
  Salmon]{ndiaye2017gap}
Eugene Ndiaye, Olivier Fercoq, Alexandre Gramfort, and Joseph Salmon.
\newblock Gap safe screening rules for sparsity enforcing penalties.
\newblock \emph{The Journal of Machine Learning Research}, 18\penalty0
  (1):\penalty0 4671--4703, 2017.

\bibitem[Nesterov(2013)]{nesterov2013gradient}
Yu~Nesterov.
\newblock Gradient methods for minimizing composite functions.
\newblock \emph{Mathematical Programming}, 140\penalty0 (1):\penalty0 125--161,
  2013.

\bibitem[Radchenko and James(2010)]{radchenko2010variable}
Peter Radchenko and Gareth~M James.
\newblock Variable selection using adaptive nonlinear interaction structures in
  high dimensions.
\newblock \emph{Journal of the American Statistical Association}, 105\penalty0
  (492):\penalty0 1541--1553, 2010.

\bibitem[She et~al.(2018)She, Wang, and Jiang]{she2018group}
Yiyuan She, Zhifeng Wang, and He~Jiang.
\newblock Group regularized estimation under structural hierarchy.
\newblock \emph{Journal of the American Statistical Association}, 113\penalty0
  (521):\penalty0 445--454, 2018.

\bibitem[Thanei et~al.(2018)Thanei, Meinshausen, and Shah]{thanei2018xyz}
Gian-Andrea Thanei, Nicolai Meinshausen, and Rajen~D Shah.
\newblock The xyz algorithm for fast interaction search in high-dimensional
  data.
\newblock \emph{The Journal of Machine Learning Research}, 19\penalty0
  (1):\penalty0 1343--1384, 2018.

\bibitem[Tibshirani et~al.(2012)Tibshirani, Bien, Friedman, Hastie, Simon,
  Taylor, and Tibshirani]{tibshirani2012strong}
Robert Tibshirani, Jacob Bien, Jerome Friedman, Trevor Hastie, Noah Simon,
  Jonathan Taylor, and Ryan~J Tibshirani.
\newblock Strong rules for discarding predictors in lasso-type problems.
\newblock \emph{Journal of the Royal Statistical Society: Series B (Statistical
  Methodology)}, 74\penalty0 (2):\penalty0 245--266, 2012.

\bibitem[Tseng(2001)]{Tseng2001}
P.~Tseng.
\newblock Convergence of a block coordinate descent method for
  nondifferentiable minimization.
\newblock \emph{Journal of Optimization Theory and Applications}, 109\penalty0
  (3):\penalty0 475--494, 2001.
\newblock ISSN 1573-2878.
\newblock \doi{10.1023/A:1017501703105}.
\newblock URL \url{http://dx.doi.org/10.1023/A:1017501703105}.

\bibitem[Wang and Ye(2014)]{wang2014two}
Jie Wang and Jieping Ye.
\newblock Two-layer feature reduction for sparse-group lasso via decomposition
  of convex sets.
\newblock In \emph{Advances in Neural Information Processing Systems}, pages
  2132--2140, 2014.

\bibitem[Wang et~al.(2013)Wang, Zhou, Wonka, and Ye]{wang2013lasso}
Jie Wang, Jiayu Zhou, Peter Wonka, and Jieping Ye.
\newblock Lasso screening rules via dual polytope projection.
\newblock In \emph{Advances in Neural Information Processing Systems}, pages
  1070--1078, 2013.

\bibitem[Wu et~al.(2010)Wu, Devlin, Ringquist, Trucco, and
  Roeder]{wu2010screen}
Jing Wu, Bernie Devlin, Steven Ringquist, Massimo Trucco, and Kathryn Roeder.
\newblock Screen and clean: a tool for identifying interactions in genome-wide
  association studies.
\newblock \emph{Genetic Epidemiology: The Official Publication of the
  International Genetic Epidemiology Society}, 34\penalty0 (3):\penalty0
  275--285, 2010.

\bibitem[Yan and Bien(2017)]{yan2017hierarchical}
Xiaohan Yan and Jacob Bien.
\newblock Hierarchical sparse modeling: A choice of two group lasso
  formulations.
\newblock \emph{Statistical Science}, 32\penalty0 (4):\penalty0 531--560, 2017.

\end{thebibliography}

\appendix

\setcounter{table}{0}
\renewcommand{\thetable}{\thesection.\arabic{table}}
\setcounter{figure}{0}
\renewcommand\thefigure{\thesection.\arabic{figure}}
\setcounter{equation}{0}
\renewcommand\theequation{\thesection.\arabic{equation}}
\setcounter{theorem}{0}
\renewcommand\thetheorem{\thesection.\arabic{theorem}}

\onecolumn 
\section*{Appendices}
\section{Does the MIP enforce Strong Hierarchy?}
We note that there can be a pathological case where the MIP does not satisfy SH. We will briefly discuss why this case corresponds to a zero probability event, when $y$ is drawn from a continuous distribution (this happens for example, if $\epsilon \sim N(0, \sigma^2)$ with $\sigma>0$). First, we assume that $\alpha_1$ and $\alpha_2$ are chosen large enough so that the number of selected variables (as counted by $\sum z_i + \sum_{i<j} z_{ij}$) is less than or equal to the number of samples $n$. We will also assume that the $X_{i}$'s and $\tilde{X}_{ij}$'s corresponding to the nonzero $z_{i}$'s and $z_{ij}$'s have full rank.

A pathological case can happen when the optimal solution of the MIP satisfies: $z^{*}_{ij}=1$, $z^{*}_{i}=1$, $\beta^{*}_{i}=0$, and $\theta^{*}_{ij} \neq 0$ for some $i$ and $j$ (for example). However, in the latter case, $\beta_i$ is a free variable, i.e., $|\beta_{i}| \leq M$ (since $z^{*}_i = 1$ and we assume $M$ to be sufficiently large). 
Thus, $\beta^{*}_i = 0$ is equivalent to saying that a least squares solution 
on the support defined by the nonzero $z^*_{i}$'s and $z^*_{ij}$'s, leads to a coordinate $\beta_{i}^*$ which is exactly zero. We know that this is a zero probability event when $y$ is drawn from a continuous distribution (assuming the number of variables is less than or equal to  the number of samples and the corresponding columns have full rank).


\section{Proof of Lemma 1}
Suppose the $z_i$'s and $z_{ij}$'s are relaxed to $[0,1]$ and fix some feasible solution $\beta, \theta$. Let us (partially) minimize the objective function with respect to $z$, while keeping $\beta, \theta$ fixed, to obtain a solution $z^{*}$. Then, $z^{*}$ must satisfy $z^{*}_{i} = \max \{ \frac{|\beta_i|}{M}, \max_{k,j \in G_i} z_{kj} \}$ for every $i$ (since this choice is the smallest feasible $z_i$). Moreover, $z^{*}_{ij} = \frac{|\theta_{ij}|}{M}$ for every $i < j$ (by the same reasoning).  Substituting the optimal values $z^{*}_i$ and $z^{*}_{ij}$ leads to
\begin{align} \label{eq:relaxation_appendix}
\min_{\beta,\theta}~~~ f(\beta, \theta) + \Omega(\beta, \theta)  \quad  \text{s.t.} \ \| \beta, \theta \|_{\infty} \leq M
\end{align}
where
$
\Omega(\beta, \theta) \eqdef \lambda_1 \sum_{i=1}^{p} \max \{ |\beta_i|, \| \theta_{G_i} \|_{\infty} \} + \lambda_2 \| \theta \|_1$
and
$ \lambda_1 = {\alpha_1}/{M}, \lambda_2 = {\alpha_2}/{M}$. Finally, we note that the box constraint in the above formulation can be removed, and the resulting formulation is still a valid relaxation.

\section{Dual Reformulation of the Proximal Problem} \label{section:dual}
In this section, we present a dual reformulation of the proximal problem, which will facilitate solving the problem. We note that \citet{jenatton2011proximal} and \citet{mairal2011convex} dualize a proximal problem which involves sum of $\ell_{\infty}$ norms (their proximal problem thus includes ours as a special case). However, the dual we will present here uses $\Theta(p^2)$ less variables as it exploits the presence of the $\ell_1$ norm in the objective. First, we introduce some necessary notation. We define the soft-thresholding operator as follows:
\begin{align*}
    \mathcal{S}_{\gamma}(\tilde{v}) = 
        \begin{cases}
        0  & \text{ if } |\tilde{v}| \leq \gamma \\
       (|\tilde{v}| - \gamma) \sign(\tilde{v}) & \text{ o.w.}
        \end{cases}
\end{align*}
We associate every $\beta^{i}$ with a dual variable $u^{i} \in \mathbb{R}$, and every $\theta_{ij}$ with two dual variables: $w^{i}_{j} \in \mathbb{R}$ and $w^{j}_{i} \in \mathbb{R}$. Moreover, we use the notation $w^{i} \in \mathbb{R}^{p-1}$ to refer to the vector composed of $w^{i}_{j}$ for all $j$ such that $j \neq i$.

\begin{theorem} {(Dual formulation)} \label{theorem:dual}
A dual of the proximal problem is:
\begin{align} \label{eq:dualproblem}
    \max_{u,w}~~q(u,w) ~~~ \text{s.t.} ~~~~~ \| (u^i, w^i) \|_1 \leq 1 \quad \forall \ i \in [p]
\end{align}
where $q(u,w)$ is a continuously differentiable function with a Lipschitz continuous gradient, and
$$\nabla_{u^i} q(u,w) = \lambda_1 \Big(  \tilde{\beta}_i - \frac{\lambda_1}{L} u^i \Big)$$   $$\nabla_{w^i_j} q(u,w) = \nabla_{w^j_i} q(u,w) = \lambda_1  \mathcal{S}_{\tfrac{\lambda_2}{L}} \big(\tilde{\theta}_{ij} - \frac{\lambda_1}{L} (w^i_j + w^j_i) \big).$$
    
If $u^{*}, w^{*}$ is a solution to~\eqref{eq:dualproblem}, then the solution to the proximal problem is:
\begin{align} \label{eq:dualtoprimal}
\beta^{*}_i = {\nabla_{u^i} q(u^{*},w^{*}) \over \lambda_1} \ \text{ and } \  \theta^{*}_{ij} = { \nabla_{w^i_j} q(u^{*},w^{*}) \over \lambda_1}.
\end{align}
\end{theorem}
\begin{proof}
Since the $\ell_1$ norm is the dual of the $l_\infty$ norm, we have:
\begin{align}\max \{ |\beta_i|, \| \theta_{G_i} \|_{\infty} \} = \max_{u^i \in \mathbb{R}, \  w^i \in \mathbb{R}^{p-1}} u^i \beta_i + \langle w^i, \theta_{G_i} \rangle  \quad \text{s.t.} \| (u^i, w^i) \|_1 \leq 1 \end{align}
Plugging the above into the proximal problem and switching the order of the min and max (which is justified by strong duality), we arrive to the dual of the proximal problem:
\begin{align} \label{eq:initialdual}
\begin{split}
   \max_{u, w} \min_{\beta,\theta} & \frac{L}{2}  \Bigg \| \begin{bmatrix} \beta - \tilde{\beta} \\ \theta - \tilde{\theta} \end{bmatrix} \Bigg \|_2^2  + \lambda_1 \sum_i (u^i \beta_i + \langle w^i,\theta_{G_i} \rangle ) + \lambda_2 \| \theta \|_1 \\
& \text{s.t. }  \| (u^i, w^i) \|_1 \leq 1, \forall i \in [p]
\end{split}
\end{align}
Note that each dual variable $u^i$ is a scalar which corresponds to the primal variable $\beta^i$. Similarly, the dual vector $w^i \in \mathbb{R}^{p-1}$ corresponds to $\theta_{G_i}$. The term $\sum_i (u^i \beta_i + \langle w^i,\theta_{G_i} \rangle )$ in \eqref{eq:initialdual} can be written as $\sum_i u^i \beta_i + \sum_{i < j} \theta_{ij} (w^i_j + w^j_i)$, where $w^i_j$ and $w^j_i$ are the components of the vectors $w^i$ and $w^j$, respectively, corresponding to $\theta_{ij}$. Using this notation, we can rewrite Problem \eqref{eq:initialdual} as follows:
\begin{align} \label{eq:maxmin}
\begin{split}
   \max_{u, w} \min_{\beta,\theta} & \sum_i h(\beta_i, u^i ;\tilde{\beta_i}) + \sum_{i < j} g(\theta_{ij}, w^i_j, w^j_i; \tilde{\theta}_{ij} ) \\
& \text{s.t. } \| (u^i, w^i) \|_1 \leq 1, \forall i \in [p]
\end{split}
\end{align}
where 
$$
h(a,b;\tilde{a}) \eqdef \frac{L}{2} (a - \tilde{a})^2 + \lambda_1 a b  \quad \text{and} \quad g(a,b,c;\tilde{a}) \eqdef \frac{L}{2} (a - \tilde{a})^2 + \lambda_1 a (b+c) + \lambda_2 |a|
$$
The optimal solution of the inner minimization in \eqref{eq:maxmin} is (uniquely) given by:
\begin{align} \label{eq:betathetastar}
\begin{split}
    \beta_i^{*} & \eqdef  \argmin_{\beta_i} h (\beta_i, u^i ;\tilde{\beta_i}) =  \tilde{\beta}_i - \frac{\lambda_1}{L} u^i \\
\theta_{ij}^{*} & \eqdef  \argmin_{\theta_{ij}} g(\theta_{ij}, w^i_j , w^j_i; \tilde{\theta}_{ij} ) = \mathcal{S}_{\frac{\lambda_2}{L}} \Big(\tilde{\theta}_{ij} - \frac{\lambda_1}{L} (w^i_j + w^j_i) \Big) 
\end{split}
\end{align}
Therefore, the dual problem can equivalently written as:
\begin{align} \label{eq:dualprox}
\begin{split}
    \max_{u,w} & \underbrace{\sum_i h(\beta_i^{*}, u^i ;\tilde{\beta_i}) + \sum_{i < j} g(\theta_{ij}^{*}, w^i_j , w^j_i; \tilde{\theta}_{ij} )}_{q(u,w)}  \quad \text{s.t.} \ \| (u^i, w^i) \|_1 \leq 1 \quad \forall \ i
\end{split}
\end{align}

Finally, since the solution $\beta^{*}, \theta^{*}$ is defined in \eqref{eq:betathetastar} is unique, Danskin's theorem implies that the dual objective function $q(u,w)$ is continuously differentiable and that 
\begin{align} \label{eq:grads}
    \nabla_{u^i} q(u,w) = \lambda_1 \beta_i^{*} \quad \text{and} \quad \nabla_{w^i_j} q(u,w) = \nabla_{w^j_i} q(u,w) = \lambda_1 \theta_{ij}^{*}.
\end{align}
\end{proof}

In problem \eqref{eq:dualproblem}, the separability of the feasible set across the $(u^i, w^i)$'s and the smoothness of $q(u,w)$ make the problem well-suited for the application of block coordinate ascent (BCA) (\cite{bertsekas2016nonlinear, Tseng2001}), which optimizes with respect to a single block at a time. When updating a particular block in BCA, we perform inexact maximization by taking a step in the direction of the gradient of the block and projecting the resultant vector onto the feasible set, i.e., the $\ell_1$ ball. We present the algorithm more formally below.
\smallskip
\begin{tcolorbox}[boxsep=1pt,left=4pt,right=4pt]
\centering
{\bf Algorithm 3: BCA for Solving \eqref{eq:dualproblem}}
\begin{itemize}[leftmargin=*]
\item Initialize with $u, w$ and take step size $\alpha_i, i \in [p]$.
\item For $i \in [p]$ perform updates (till convergence):
\begin{equation*}
\begin{bmatrix} u^i \\ w^i \end{bmatrix} \gets \mathcal{P}_{\| . \|_1 \leq 1} \Bigg(\begin{bmatrix} u^i \\ w^i \end{bmatrix} + \alpha_i \nabla_{u^i, w^i} q(u,w) \Bigg),
\end{equation*}
where for a vector $a$, $\mathcal{P}_{\| . \|_1 \leq 1}(a)$ denotes projection of $a$ onto the 
unit $\ell_{1}$-ball.
\end{itemize}
\end{tcolorbox}
The Lipschitz parameter of $\nabla_{u^i, w^i} q(u,w)$ is given by $L_i = p \frac{\lambda_1^2}{L}$ (this follows by observing that each component of $\nabla_{u^i, w^i} q(u,w)$ is a piece-wise linear function with a maximal slope of $\frac{\lambda_1^2}{L}$). Thus, by standard results on block coordinate descent (e.g., \citet{BeckConvergence, bertsekas2016nonlinear}), Algorithm 3 with step size $\alpha_i = \frac{1}{L_i}$ converges at a rate of $\mathcal{O}(\frac{1}{t})$ (where $t$ is the iteration counter). We note that BCA has been applied to the dual of structured sparsity problems (e.g., \cite{jenatton2011proximal, yan2017hierarchical})---however, the duals considered in the latter works are different.

\section{Proof of Theorem 1}
We prove the theorem using the dual reformulation presented in Theorem \eqref{theorem:dual} and the block coordinate ascent (BCA) algorithm presented in Section \ref{section:dual}. Suppose $\sum_{t \in G_i} \max \{ | \tilde{\theta}_{t} | - \frac{\lambda_2}{L}, 0\} \leq \frac{\lambda_1}{L} - | \tilde{\beta}_i |$ is satisfied for some $i$. Let $u, w$ be some feasible solution to Problem \eqref{eq:dualproblem}  (e.g., solution of all zeros). Now update $u, w$ as follows:
\begin{align} \label{eq:choice1}
    u^i = \frac{L}{\lambda_1} \tilde{\beta}_i,
\end{align}
and for every $t \in G_i$, let $j$ be the index in $t$ different from $i$ and set:
\begin{align}\label{eq:choice2}
    w^{i}_{j} =  \max \Big \{\frac{L}{\lambda_1} |\tilde{\theta}_{t}| - \frac{\lambda_2}{\lambda_1}, 0 \Big \} \sign(\tilde{\theta}_{t}) \quad \text{ and } \quad w^{j}_{i} = 0
\end{align}
It is easy to check that $u, w$ is still feasible for Problem \eqref{eq:dualproblem} after this update and that $\nabla_{u^i, w^i} q(u,w) = 0$ and $ \nabla_{w^j_i} q(u,w) = 0$ for every $j$. Thus, BCA will never change $u^i$, $w^i$, or $w^j_i$ (for any $j$) in subsequent iterations. Since BCA is guaranteed to converge to an optimal solution, we conclude that the values in \eqref{eq:choice1} and \eqref{eq:choice2} (which correspond to $\beta^{*}_i, \theta^{*}_{G_i} = 0$) are optimal.

For the case when $|\tilde{\theta}_{ij}| \leq \frac{\lambda_2}{L}$, if we set $w^{i}_{j} = w^{j}_{i} = 0$, then  $\nabla_{w^i_j} q(u,w) = \nabla_{w^j_i} q(u,w) = 0$, so BCA will never change  $w^{i}_{j}$ or $ w^{j}_{i}$, which leads to $\theta^{*}_{ij} = 0$.

\section{Proof of Theorem 2}
By Theorem 1, we have $\beta_{\mathcal{V}^c}^{*} = 0,~~\theta_{\mathcal{E}^c}^{*} = 0$. Plugging the latter into the proximal problem leads to:
\begin{align} 
\begin{bmatrix} \beta^{*}_{\mathcal{V}} \\ \theta^{*}_{\mathcal{E}} \end{bmatrix} = \argmin_{ \beta_{\mathcal{V}}, \theta_{\mathcal{E}}} \frac{L}{2} \Bigg \| \begin{bmatrix} \beta_{\mathcal{V}} - \tilde{\beta}_{\mathcal{V}} \\ \theta_{\mathcal{E}} - \tilde{\theta}_{\mathcal{E}} \end{bmatrix} \Bigg \|_2^2 +  \Omega(\beta_{\mathcal{V}}, \theta_{\mathcal{E}})
\end{align}
where $\Omega(\beta_{\mathcal{V}}, \theta_{\mathcal{E}}) = \lambda_1 \sum_{i \in \mathcal{V}} \max \{ |\beta_i|, \| \theta_{\tilde{G_i}} \|_{\infty} \} + \lambda_2 \| \theta_{\mathcal{E}} \|_1$ and $\tilde{G_i} = G_i \setminus \mathcal{E}^c$. But by the definition of the connected components, the following holds:
$$\Omega(\beta_{\mathcal{V}}, \theta_{\mathcal{E}}) = \sum_{l \in [k]} \Big[ \lambda_1 \sum_{i \in \mathcal{V}_l} \max \{ |\beta_i|, \| \theta_{\tilde{G_i}} \|_{\infty} \} + \lambda_2 \| \theta_{\mathcal{E}_l} \|_1  \Big] = \sum_{l \in [k]} \Omega(\beta_{\mathcal{V}_l}, \theta_{\mathcal{E}_l}).$$
The above implies that the proximal problem is separable across the blocks $\beta_{\mathcal{V}_l}, \theta_{\mathcal{E}_l}$, which leads to the result of the theorem.

\section{Proof of Lemma 2}
First, note that the full gradient $\nabla_{\beta,\theta} f(\hat{\beta},\hat{\theta})$ is sufficient for constructing $\mathcal{G}$ (see steps 2 and 3 of Algorithm 2). The $(i,j)$'s in $\mathcal{T}^c$ whose $|\tilde{\theta}_{ij}| \leq \lambda_2/L$ are not needed to construct $\mathcal{G}$ (this follows from the definitions of $\mathcal{V}$ and $\mathcal{E}$). For every $(i,j) \in \mathcal{T}^c$, we have $\hat{\theta}_{ij} = 0$, so the  condition $|\tilde{\theta}_{ij}| \leq \lambda_2/L$ is equivalent to $|\nabla_{\theta_{ij}} f(\hat{\beta}, \hat{\theta})| \leq \lambda_2$ (see the definition of $\tilde{\theta}_{ij}$ in step 2 of Algorithm 2). Thus, the $(i,j)$'s in $\mathcal{T}^c$ with $|\nabla_{\theta_{ij}} f(\hat{\beta}, \hat{\theta})| \leq \lambda_2$ are not needed to construct $\mathcal{G}$. The latter indices are exactly those in $\mathcal{S}^c$. Thus, the remaining parts of the gradient are: $\nabla_{\beta} f(\hat{\beta}, \hat{\theta})$, $\nabla_{\theta_{\mathcal{T}}} f(\hat{\beta}, \hat{\theta})$, and $\nabla_{\theta_{\mathcal{S}}} f(\hat{\beta}, \hat{\theta})$---these are sufficient to construct $\mathcal{G}$.

\section{Proof of Lemma 3}
Let $(i,j) \in \mathcal{S}$. By the triangle inequality:
\begin{align} \label{eq:screenproofinitial}
     |\nabla_{\theta_{ij}} f(\hat{\beta}, \hat{\theta})| & \leq  |\nabla_{\theta_{ij}} f(\beta^{w}, \theta^{w})  | + |  \nabla_{\theta_{ij}} f(\hat{\beta}, \hat{\theta}) - \nabla_{\theta_{ij}} f(\beta^{w}, \theta^{w})   |
\end{align} 
Writing down the gradients explicitly and using Cauchy-Schwarz, we get: \begin{align*}
| \nabla_{\theta_{ij}} f(\hat{\beta}, \hat{\theta}) - \nabla_{\theta_{ij}} f(\beta^{w}, \theta^{w})| & = | \tilde{X}_{ij}^T (y - X \hat{\beta} - \tilde{X} \hat{\theta}) - \tilde{X}_{ij}^T (y - X \beta^{w} - \tilde{X} \theta^{w}) | \\
& \leq \| \tilde{X}_{ij} \|_2 \| (X \hat{\beta} + \tilde{X} \hat{\theta}) - (X \beta^{w} + \tilde{X} \theta^{w}) \|_2 \\
& \leq C \| \gamma \|_2
\end{align*}
Plugging the upper bound above into \eqref{eq:screenproofinitial}, we get $|\nabla_{\theta_{ij}} f(\hat{\beta}, \hat{\theta})| \leq |\nabla_{\theta_{ij}} f(\beta^{w}, \theta^{w})  | +  C \| \gamma \|_2$. Therefore, if $(i,j) \in \mathcal{S}$, i.e., $|\nabla_{\theta_{ij}} f(\hat{\beta}, \hat{\theta})| > \lambda_2$ then $|\nabla_{\theta_{ij}} f(\beta^{w}, \theta^{w})  | + C \| \gamma \|_2 > \lambda_2$, implying that $(i,j) \in \hat{\mathcal{S}}$.

\section{Results of Additional Experiments}

\subsection{Sizes of Connected Components}
For the Riboflavin ($p = 4088$, $n=71$) and Coepra ($p = 5786$, $n=89$) datasets (discussed in the paper), we fit a regularization path with 100 solutions using Algorithm 2. In Table \ref{table:sizes}, we report the maximum number of edges and vertices encountered across all the connected components and for all the 100 solutions in the path.
\begin{table}[h]
\centering
\caption{Maximum size of the connected components across a regularization path of 100 solutions.}
\label{table:sizes}
\begin{tabular}{|l|cc|cc|}
\hline
\multirow{2}{*}{Dataset} & \multicolumn{2}{c|}{$\lambda_2 = 2 \lambda_1$} & \multicolumn{2}{c|}{$\lambda_2 = \lambda_1$} \\ \cline{2-5} 
                         & \# Edges             & \# Vertices             & \# Edges            & \# Vertices            \\ \hline
Ribo                     & 350                  & 149                     & 1855                & 693                    \\
Coepra                   & 227                  & 86                      & 400                 & 103                    \\ \hline
\end{tabular}
\end{table}

\subsection{Prediction Error}
\begin{figure}[H]
\centering
 \includegraphics[scale=0.5]{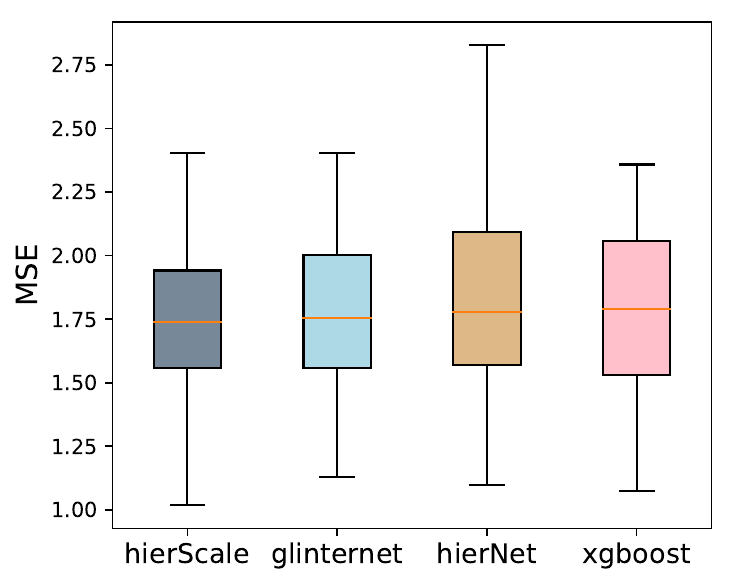}
  \caption{MSE on the test set for synthetic data (Anti-Hierarchical truth).}
\end{figure}

\begin{figure}[H]
\centering
 \includegraphics[scale=0.5]{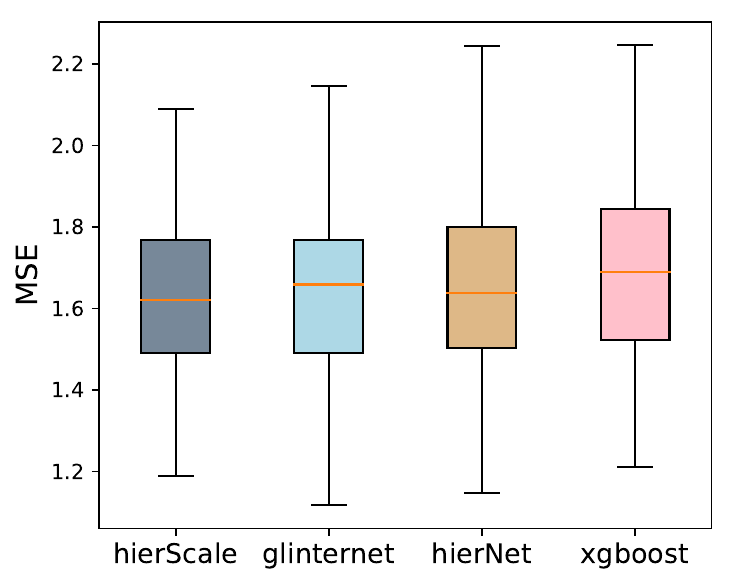}
  \caption{MSE on the test set for synthetic data (Hierarchical truth).}
\end{figure}

\end{document}